\documentclass[submission,copyright,creativecommons]{eptcs}
\providecommand{\event}{M4M 2017} % Name of the event you are submitting to
\usepackage{breakurl}             % Not needed if you use pdflatex only.
\usepackage{underscore}           % Only needed if you use pdflatex.
\usepackage{amsmath, amsthm}
\usepackage{graphicx}

\title{Cooperative Epistemic Multi-Agent Planning \\ for Implicit Coordination}
\author{Thorsten Engesser
\institute{Institut f\"ur Informatik \\ Albert-Ludwigs-Universit\"at \\ Freiburg, Germany}
\email{engesset@cs.uni-freiburg.de}
\and
Thomas Bolander 
\institute{DTU Compute \\ Technical University of Denmark \\ Copenhagen, Denmark}
\email{tobo@dtu.dk}
\and
Robert Mattm\"uller
\institute{Institut f\"ur Informatik \\ Albert-Ludwigs-Universit\"at \\ Freiburg, Germany}
\email{mattmuel@cs.uni-freiburg.de}
\and
Bernhard Nebel
\institute{Institut f\"ur Informatik \\ Albert-Ludwigs-Universit\"at \\ Freiburg, Germany}
\email{nebel@cs.uni-freiburg.de}
}
\def\titlerunning{Cooperative Epistemic Multi-Agent Planning for Implicit Coordination}
\def\authorrunning{T. Engesser, T. Bolander, R. Mattm\"uller \& B. Nebel}

\usepackage{color, url}
\usepackage{pgf, tikz}
\usepackage{enumerate}
\usepackage{paralist}
\usepackage{subcaption}

\def\myiff{\textnormal{ iff }}
\def\myand{\textnormal{ and }}

\newcommand{\myset}[1]{\left\{#1\right\}}
\newcommand{\mysetw}[2]{\myset{#1~|~#2}}

\newcommand{\agents}{\mathcal{A}}
\newcommand{\model}{\mathcal{M}}
\newcommand{\myevent}{\mathcal{E}}

\def\lang{\mathcal{L}_{\textnormal{KC}}} %(P,\agents)}
\def\langd{\mathcal{L}_{\textnormal{DEL}}} %(P,\agents)}

\theoremstyle{definition}
\newtheorem{definition}{Definition}
\newtheorem{example}{Example}

\theoremstyle{plain}
\newtheorem{theorem}{Theorem}

\newtheorem{proposition}[theorem]{Proposition}

\def\pre{\mathit{pre}}
\def\post{\mathit{post}}
\def\owner{\omega}
\def\succ{\sigma}

\def\tpl{\left<}
\def\tpr{\right>}

\newcommand{\after}[1]{(\!(#1)\!)} % \llparenthesis #1 \rrparenthesis
\newcommand{\als}[2]{{#1}^{#2}}
\newcommand{\kafter}[1]{K_{\owner(#1)}\after{#1}} % TODO. First Idea: {K_{\after{#1}}}

\def\globals{\textit{Globals}}
\newcommand{\smin}[1]{S_#1^\text{min}}
\newcommand{\sglobal}{S^\text{gl}}

\let\phi\varphi

\def\at{\textit{at-}}
\def\for{\textit{for-}}

\usetikzlibrary{trees,arrows,arrows,shapes,positioning,calc}
\tikzset{
  uworld/.style={circle, fill, inner sep=1.5pt, outer sep=2pt, anchor=base},
  udworld/.style={circle, draw, inner sep=2.7pt, outer sep=2pt, anchor=base}}

\newcommand{\cfig}[1]{\begingroup
\setbox0=\hbox{#1}%
\parbox{\wd0}{\box0}\endgroup}

\newcommand{\onestate}[3]{%
\begin{tikzpicture}[baseline=-4pt] \footnotesize
    \draw
    (0,0) node[udworld, #3] (d1) {}
    (0,0) node[uworld, #2, label=below:{#1}] (w1) {};
\end{tikzpicture}}

\newcommand{\twoindist}[9]{%
\begin{tikzpicture}[baseline=-3.5pt] \small %\footnotesize
    \draw
    (0,0) node[udworld, #6] (d1) {}
    (#9,0) node[udworld, #7] (d2) {}
    (0,0) node[uworld, #4, label=below:{#1}] (w1) {}
    (#9,0) node[uworld, #5, label=below:{#2}] (w2) {};
    \draw[#8] (w1) -- node[above] {#3} (w2);
\end{tikzpicture}}

\newcommand{\leaveout}[1]{}

\DeclareMathAlphabet{\mathcal}{OMS}{cmsy}{m}{n}

\begin{document}
\maketitle

\begin{abstract}
Epistemic planning can be used for decision making in multi-agent situations
with distributed knowledge and capabilities. Recently, Dynamic Epistemic Logic (DEL) has been shown to provide a very natural and expressive framework for epistemic planning.
We extend the DEL-based epistemic planning framework to include perspective shifts, allowing us to define new
notions of sequential and conditional planning with implicit coordination. With
these, it is possible to solve planning tasks with joint goals in a decentralized
manner without the agents having to negotiate about and commit to a joint policy
at plan time. First we define the central planning notions and sketch the implementation of a planning system built on those notions.
Afterwards we provide some case studies in order to evaluate the planner
empirically and to show that the concept is useful for multi-agent systems in practice.
\end{abstract}

% !TEX root =  paper.tex

\section{Introduction}
% \enlargethispage{1\baselineskip}
One important task in Multi-Agent Systems is to collaboratively reach a joint
goal with multiple autonomous agents. The problem is particularly challenging in
situations where the knowledge and capabilities required to reach the goal are
distributed among the agents.  Most existing approaches therefore apply some
centralized coordinating instance from the outside, strictly separating the
stages of communication and negotiation from the agents' internal planning and
reasoning processes.  In contrast, building upon the epistemic planning
framework by Bolander and Andersen~\cite{bolander11}, we propose a decentralized
planning notion in which each agent has to individually reason about the entire
problem and autonomously decide when and how to (inter-)act. For this, both
reasoning about the other agents' possible contributions and reasoning about
their capabilities of performing the same reasoning is needed.
We achieve our notion of implicitly coordinated plans by requiring all
desired communicative abilities to be modeled as epistemic actions
which then can be planned alongside their ontic counterparts, thus
enabling the agents to perform observations and coordinate at run time.
%While this imposes certain restrictions on the problems that can be solved,
It captures the intuition that communication clearly constitutes an
action by itself and, more subtly, that even a purely ontic action can
play a communicative role (e.g.\ indirectly suggesting follow-up
actions to another agent).  Thus, for many problems our approach
appears quite natural. On the practical side, the epistemic planning
framework allows a very expressive way of defining both the agents'
physical and communicative abilities.

Consider the following example scenario. Bob would like to borrow the apartment
of his friend Anne while she is away on vacation. Anne would be very happy to do
him this favor. So they now have the joint goal of making sure that Bob can
enter the apartment when he arrives. Anne will think about how to achieve the
goal, and might come up with the following plan: Anne puts the key under the
door mat; when Bob arrives, Bob takes the key from under the door mat; Bob opens
the door with the key. Note that the plan does not only contain the actions
required by Anne herself, but also the actions of Bob. These are the kind of
multi-agent plans that this paper is about.

However, the plan just presented does not count as an \emph{implicitly
coordinated} plan. When Bob arrives at the apartment, he will clearly not know
that the key is under the door mat, unless Anne has told him, and this
announcement was not part of the plan just presented. If Anne has the ability to
take Bob's perspective (she has a \emph{Theory of Mind} concerning Bob
\cite{prem.ea:does}), Anne should of course be able to foresee this problem, and
realize that her plan can not be expected to be successful. An improved plan
would then be: Anne puts the key under the door mat; Anne calls Bob to let him
know where the key is; when Bob arrives, Bob takes the key from under the door
mat; Bob opens the door with the key. This \emph{does} qualify as an implicitly
coordinated plan. Anne now knows that Bob will know that he can find the key
under the door mat and hence will be able to reach the goal. Anne does not have
to request or even coordinate the sub-plan for Bob (which is: take key under
door mat; open door with key), as she knows he will himself be able to determine
this sub-plan given the information she provides. This is an important aspect of
implicit coordination: coordination happens implicitly as a consequence of
observing the actions of others (including announcements), never explicitly
through agreeing or committing to a specific plan. The essential contributions
of this paper are to formally define this notion of implicitly coordinated plans
as well as to document and benchmark an implemented epistemic planner that
produces such plans.

\enlargethispage{\baselineskip}
Our work is situated in the area of distributed problem solving and
planning~\cite{durfee01} and directly builds upon the framework introduced by
Bolander and Andersen~\cite{bolander11} and L\"owe, Pacuit, and Witzel~\cite{lowe11}, who formulated the planning problem in the context of
Dynamic Epistemic Logic (DEL)~\cite{ditmarsch07}. Andersen, Bolander, and
Jensen~\cite{andersen12} extended the approach to allow strong and weak
conditional planning in the single-agent case. Algorithmically, (multi-agent)
epistemic planning can be approached either by compilation to classical
planning~\cite{albore09,kominis15,muise15} or by search in the space of
``nested''~\cite{bolander11} or ``shallow'' knowledge
states~\cite{petrick02,petrick04,petrick13}. Since compilation approaches to
classical planning can only deal with bounded nesting of knowledge (or belief),
similar to Bolander and Andersen~\cite{bolander11}, we use search in the space
of epistemic states to find a solution. One of the important features that distinguishes our work from more traditional multi-agent planning~\cite{brenner09} is the explicit notion of perspective shifts needed for agents to reason about the possible plan contributions of other agents---and hence needed to achieve implicit coordination. %These perspective shifts are essential to our type of decentralized planning, where coordination is implicit (achieved by observing the actions of others).

Our concepts can be considered related to recent work in temporal
epistemic logics \cite{bulling14,jamroga07,jamroga04b}, which
addresses a question similar to ours, namely what groups of agents can jointly
achieve under imperfect information. These approaches are based on concurrent
epistemic game structures. Our approach is different in a number of ways, including: 1) As in classical planning, our actions and their effects are explicitly and compactly represented in an action description language (using the event models of DEL); 2) Instead of joint actions we have sequential action execution, where the order in which the agents act is not predefined; 3) None of the existing solution concepts considered in temporal epistemic logics capture the stepwise shifting perspective underlying our notion of implicitly coordinated plans.

The present paper is an extended and revised version of a paper presented at the Workshop on Distributed and Multi-Agent Planning (DMAP) 2015 (without archival proceedings). The present paper primarily offers an improved presentation: extended and improved introduction, improved motivation, better examples, improved formulation of definitions and theorems, simplified notation, and more discussions of related work. The technical content is essentially as in the original DMAP paper, except we now compare implicitly coordinated plans to standard sequential plans,  we now formally derive the correct notion of an implicitly coordinated plan from the natural conditions it should satisfy, and we have added a proposition that gives a semantic characterisation of implicitly coordinated policies. % (Proposition~\ref{prop:strong}).

% (where the whole transition system is given as a
%single model). In each state of the system the subset of agents that have an
%influence on the successor state is predefined. The actual transition is then
%given by these agents' individual, simultaneous actions. This is in contrast to
%our approach, where the agent to perform the next sequential action is not
%predefined, which
% leads to the need of implicit scheduling and coordination.
%makes decentralized planning and execution via implicit coordination possible in the first place.
% Some thoughts:
% Implicit coordination is needed if both planning and acting is to be
% decentralized! shifting perspective is insufficient if there are joint actions
% (BoS style coordination problems)! => only sequential, individual actions can
% be coordinated implicitly! (and we have to deal with dead-/livelocks...)
%Thorsten's version:
% Finally, our concepts can be considered related to approaches from
% temporal epistemic logics
% \cite{jamroga07,jamroga07,jamroga07,jamroga07,jamroga07}, since our planning
% tasks constitute a compact way to induce transition systems that are similar to
% concurrent epistemic game structures. While it seems possible (yet nontrivial)
% to extend solution concepts from temporal epistemic logics to allow for some
% sort of implicit coordination, we argue that our more relaxed setting (in which
% actions are performed non-concurrently but without a preimposed ordering on the
% agents) is more appropriate for the kind of problem solving we're interested
% in.}
  % Introduction
% !TEX root =  paper.tex

\section{Theoretical Background}
To represent planning tasks as the `apartment borrowing' example of the introduction, we need a formal framework where: (1) agents can reason about the knowledge and ignorance of other agents; (2) both fully and partially observable actions can be described in a compact way (Bob doesn't see Anne placing the key under the mat). Dynamic Epistemic Logic (DEL) satisfies these conditions. We first briefly recapitulate the foundations of DEL, following the conventions of Bolander and Andersen~\cite{bolander11}. What is new in this exposition is mainly the parts on perspective shifts in Section~\ref{sect:perspectiveshifts}.

We now define epistemic languages, epistemic states and epistemic actions. All of these are defined relative to a given finite set of \emph{agents} $\agents$ and a given finite set of \emph{atomic propositions} $P$. To keep the exposition simple, we will not mention the dependency on $\agents$ and $P$ in the following.

\subsection{Epistemic Language and Epistemic States}
 %This means that our version of DEL includes postconditions allowing for ontic/factual change, but postconditions are without loss of expressivity limited to conjunctions of literals (as in classical planning).
\begin{definition}
  The {\em epistemic language} $\lang$ is
  $\phi ~::=~ \top ~|~ \bot ~|~ p ~|~ \neg \phi ~|~ \phi \land \phi ~|~ K_i \phi ~|~ C \phi$
  where $p \in P$ and $i \in \agents$.
\end{definition}
We read $K_i \phi$ as ``agent $i$ knows $\phi$'' and $C \phi$ as ``it
is common knowledge that $\phi$''. %In the
%following, we will always use $P$ to denote our set of atomic
%propositions, and $\agents$ our set of agents.
%Formulas of the
%epistemic language are evaluated in epistemic models.

\begin{definition}
  An {\em epistemic model} is $\model =
  \tpl W, (\sim_i)_{i \in \agents}, V \tpr$ where
  \begin{itemize}
  \item The \emph{domain} $W$ is a non-empty finite set of
    \emph{worlds}.
  \item $\sim_i\ \subseteq W \times W$ is an equivalence relation called
    the \emph{indistinguishability relation} for agent $i$.
  \item $V : P \to \mathcal{P}(W)$ assigns a \emph{valuation} to each atomic
    proposition.
  \end{itemize}

  For $W_d \subseteq W$, the pair $(\model,W_d)$ is called an
  \emph{epistemic state} (or simply a \emph{state}), and the worlds of
  $W_d$ are called the \emph{designated worlds}. A state is
  called \emph{global} if $W_d = \{ w \}$ for some world $w$ (called the \emph{actual world}), and we then often write $(\model,w)$ instead of $(\model, \{w\})$. We use $\sglobal$ to denote the set of global states. For any state $s = (\model, W_d)$, we let $\globals(s) = \{ (\model, w) \mid w \in W_d \}$.
 % In general, $(\model, W_d)$ can be thought of as the belief state  $\mysetw{(\model,\myset{w})}{w \in W_d}$ over possible global  states.
  A
  state $(\model,W_d)$ is called a \emph{local state} for agent $i$ if
  $W_d$ is closed under $\sim_i$. A local state for $i$ is \emph{minimal} if
  $W_d$ is a minimal set closed under $\sim_i$. We use $\smin{i}$ to denote the set of minimal local states of $i$.
    Given a state $s = (\model, W_d)$, the \emph{associated local state} of agent $i$,
  denoted $s^i$, is $(\model, \{ v \mid v \sim_i w \text{ and } w \in W_d
  \})$.
\end{definition}

\def\nls{\\}%[-0.5mm]}
\begin{definition}
  Let $(\model,W_d)$ be a state with $\model = \tpl W,
  (\sim_i)_{i \in \agents}, V \tpr $. For $i \in \agents$, $p \in P$ and
  $\phi, \psi \in \lang$, we define truth as follows (with the propositional cases being standard and hence left out):
  \begin{align*}
    (\model, W_d) &\models \phi &\myiff &(\model,w) \models \phi\ \text{ for all } w \in W_d \nls
  %  (\model, w) &\models p &\myiff& w \in V(p) \nls
  %  (\model, w) &\models \neg \phi &\myiff& \model,w \not\models \phi \nls
  %  (\model, w) &\models \phi \land \psi &\myiff& \model,w \models \phi \myand \model,w \models \psi \nls
    (\model, w) &\models K_i \phi &\myiff& (\model,w') \models \phi\ \text{ for all } w' \sim_i w \nls
    (\model, w) &\models C\phi &\myiff& \model,w' \models \phi\ \text{ for all } w' \sim^\ast w
  \end{align*}
  where $\sim^\ast$ is the transitive closure of\ \ $\bigcup_{i \in \agents} \sim_i$.
\end{definition}

\begin{example}\label{exam:states} \upshape
 Let $\agents = \{ \textit{Anne}, \textit{Bob} \}$ and $P = \{ m \}$, where $m$ is intended to express that the key is under the door mat.
%  A global state describes an epistemic situation from a
%  global perspective, where the actual world has been pointed
 % out.
 Consider the following global state $s = (\model,w)$, where
  the nodes represent worlds, the edges represent the
  indistinguishability relations (reflexive edges left out), and
  \raisebox{0pt}{\scalebox{1}{\onestate{}{}{}}} is used for designated worlds:
\begin{center}
  $s = (\model, w) = \twoindist{$w:m$}{$v:$}{\itshape Bob}{}{}{}{white}{}{2}$
\end{center}
  %
%  We use \scalebox{0.7}{$\begin{tikzpicture} \node[udworld]{};
 %     \node[uworld]{}; \end{tikzpicture}$} to denote the designated
%  worlds, in this case only $w_1$.
Each node is labeled by the name of the world, and the list of atomic
propositions true at the world. The state represents a situation where the key
is under the door mat ($m$ is true at the actual world $w$), but Bob considers
it possible that it isn't ($m$ is not true at the world $v$ indistinguishable
from $w$ by Bob). We can verify that in this state Anne knows that the key is
under the mat, Bob doesn't, and Anne knows that he doesn't: $s \models
K_{\textit{Anne}} m \wedge \neg K_{\textit{Bob}} m \wedge K_\textit{Anne} \neg
K_\textit{Bob} m$. The fact that Bob does not know the key to be under the mat
can also be expressed in terms of local states. Bob's local perspective on the
state $s$ is his associated local state $s^\textit{Bob} = (\model, \{w,v\})$. We
have $s^\textit{Bob} \not\models m$, signifying that from Bob's perspective, $m$
can not be verified.
% not be verified that the key is under the mat, which is signified by having both $s^\textit{Bob} \not\models m$ and $s^\textit{Bob} \not\models \neg m$.
%The actual world is $w_1$ where $p$
%  holds, but for agent 1 (and 2) the actual world $w_1$ is
%  indistinguishable from the world $w_2$ where $p$ is false. Since
%  agent 1 is ignorant about whether the actual world is $w_1$ or
%  $w_2$, the model that represents his local view on the situation is
 % $(\model,\{w_1,w_2 \})$, which is exactly his associated local state
%  of $(\model,\{w_1\})$. The point is that since agent 1 is unable to
%  point out whether the actual world is $w_1$ or $w_2$, his internal
 % state must consider both as candidates for being the actual world,
%  and this is exactly what the model $(\model, \{w_1,w_2 \}) =
%  (\model,\{w_1\})^1$ does.  We have $(\model, \{w_1\})^1 \not\models
%  p$ and $(\model, \{w_1\})^1 \not\models \neg p$, corresponding to
%  the fact that from agent 1's local perspective it can not be
 % verified whether $p$ holds or not.
\end{example}

\subsection{Perspective Shifts}\label{sect:perspectiveshifts}

%In general, given a state $s$, the associated local state $s^i$ will
%represent agent $i$'s internal perspective on that state. Going from $s$ to
%$s^i$ amounts to a \emph{perspective shift}, where the perspective is shifted to
%the local perspective of agent $i$}
In general, given a global state $s$, the associated local state $s^i$
will represent agent $i$'s internal perspective on that state. Going from
$s^i$ to $(s^i)^j$ amounts to a \emph{perspective shift}, where agent $i$'s
perspective on the perspective of agent $j$ is taken.
In Example~\ref{exam:states}, Anne's
perspective on the state $s$ is $s^\textit{Anne}$, which is $s$ itself. Bob's
perspective is $s^{\textit{Bob}} = (\model, \{w,v\})$. When Anne wants to reason
about the available actions to Bob, e.g. whether he will be able to take the key
from under the door mat or not, she will have to shift to his perspective, i.e.
reason about what holds true in
$(s^\textit{Anne})^\textit{Bob}$, which is the same as $s^\textit{Bob}$ in
this case. This type of perspective shift
is going to be central in our notion of implicitly coordinated plans, since it
is essential to the ability of an agent to reason about other agents' possible
contributions to a plan \emph{from their own perspective}. As the introductory
example shows, this ability is essential: If Anne can not reason about Bob's
contribution to the overall plan from his own perspective, she will not realize
that she needs to call him to let him know where the key is.

Note that any local state $s$ induces a unique set of global states, $\globals(s)$, and that we can hence choose to think of $s$ as a compact representation of the ``belief state'' $\globals(s)$ over global states.

We have the following basic properties concerning perspective shifts (associated local states), where the third follows
directly from the two first:

\begin{proposition}\label{lemma:localisation}
  Let $s$ be a state, $i \in
  \agents$ and $\phi \in \lang$.
  \begin{enumerate}
  \item $s^i \models \phi$ iff $s \models K_i
    \phi$.
  \item If $s$ is local for agent $i$ then $s^i
    = s$.
  \item If $s$ is local for agent $i$ then $s
    \models \phi$ iff $s \models K_i \phi$.\qed
  \end{enumerate}
\end{proposition}

%Perspective shifts are of fundamental importance in multi-agent planning to allow an agent to reason about the other agents' possible contributions to a plan.

\subsection{Dynamic Language and Epistemic Actions}
To model actions, like announcing locations of keys and picking them up, we use the \emph{event models} of DEL.

\begin{definition}
  An {\em event model} is $\myevent = \tpl E, (\sim_i)_{i \in
    \agents}, \pre, \post \tpr$ where
  \begin{itemize}
  \item The \emph{domain} $E$ is a non-empty finite set of \emph{events}.
  \item $\sim_i\ \subseteq E \times E$ is an equivalence relation called
    the \emph{indistinguishability relation} for agent $i$.
  \item $\pre : E \to \lang$ assigns a \emph{precondition} to each
    event.
  \item $\post : E \to \lang$ assigns a \emph{postcondition} to each
    event. For all $e\in E$, $post(e)$ is a conjunction of literals
    (atomic propositions and their negations, including $\top$).
  \end{itemize}%\newpage
  For $E_d \subseteq E$, the pair $(\myevent,E_d)$ is called an
  \emph{epistemic action} (or simply an \emph{action}), and the events in
  $E_d$ are called the \emph{designated events}. An action is called \emph{global} if $E_d = \{e \}$ for some event $e$, and we then often write $(\myevent,e)$ instead of $(\myevent, \{ e \})$.
   Similar to
  states, $(\myevent,E_d)$ is called a \emph{local action} for agent $i$
  when $E_d$ is closed under $\sim_i$.
\end{definition}

Each event of an epistemic action represents a different possible
outcome. By using multiple events $e, e' \in E$ that are
indistinguishable (i.e. $e\sim_ie'$), it is possible to obfuscate the
outcomes for some agent $i \in \agents$, i.e. modeling partially
observable actions. Using event models with $| E_d | > 1$, it is also
possible to model sensing actions (where a priori, multiple outcomes are
considered possible), and nondeterministic actions~\cite{bolander11}.

The {\em product update} is used to specify the successor state
resulting from the application of an action in a state.

\begin{definition}
  Let a state $s = (\model, W_d)$ and an action $a = (\myevent, E_d)$ be given with $\model = \tpl W, (\sim_i)_{i \in \agents}, V\tpr$ and
  $\myevent = \tpl E, (\sim_i)_{i\in \agents}, \pre, \post \tpr$. Then the
  \emph{product update} of $s$ with $a$ is $s \otimes a  = (\tpl W', (\sim'_i)_{i \in \agents}, V' \tpr, W'_d)$ where
  \begin{itemize}
  \item $W' = \mysetw{(w,e) \in W \times E}{\model, w \models
    \pre(e)}$
  \item $\sim'_i\ = \mysetw{((w,e),(w',e')) \in W' \times W'}{w\sim_iw'
    \myand e \sim_ie'}$
  \item $V'(p) = \{ (w,e) \in W' \mid \post(e) \models p \text{ or }
    (\model,w \models p \ \myand \post(e) \not\models \neg p) \}$
  \item $W'_d = \mysetw{(w,e) \in W'}{w \in W_d \text{ and } e \in E_d}$.
  \end{itemize}
\end{definition}
%If both $s$ and $a$ are local for agent $i$, then so is $s \otimes a$~\cite{bolander11}.
\begin{example}\label{exam:take} \upshape
  Consider the following epistemic action $\textit{try-take} = (\myevent, \{e , f \})$,
  using the same conventions as for epistemic models, except each
  event $e$ is labeled by $\langle \pre(e), \post(e) \rangle$:
  \begin{center}
  $\textit{try-take} = (\myevent,\{e,f\}) =
  \twoindist{$e:\langle m, h \land \neg m \rangle$}
            {$f: \langle \neg m, \top \rangle$}{\itshape Anne}{}{}{}{}{}{2.5}$
  \end{center}
  It represents the action of Bob attempting to take the key  from under the mat. The event $e = \langle m, h \land \neg m \rangle$ represents that if the key is indeed under the mat (the precondition $m$ is true), then the result will be that Bob holds the key and it is no longer under the mat (the postcondition $h \land \neg m$ becomes true). The event $f = \langle \neg m, \top \rangle$ represents that if the key is not under the mat, nothing will happen (the postcondition is the trivial one, $\top$). Note the indistinguishability edge for Anne: She is not there to see whether the action is successful or not. Note however that she is still is aware that either $e$ or $f$ happens, so the action represents the situation where she knows that he is attempting to take the key, but not necessarily whether he is successful (except if she already either knows $m$ or knows $\neg m$). 

Letting $s$ denote the state from Example~\ref{exam:states}, we can calculate the result of executing \textit{try-take} in $s$:
\begin{center}
  $s \otimes \textit{try-take} =
  \twoindist{$(w,e):h$}
            {$(v,f):$}{}{}{}{}{white}{white}{2.5}$
\end{center}
  Note that the result is for Bob to have the key and for this to be common knowledge among Anne and Bob ($s \otimes \textit{try-take} \models C h$). So it seems that if we assume initially to be in the state $s$, and want to find a plan to achieve $h$, then the simple plan consisting only of the action \textit{try-take} should work. It is, however, a bit more complicated than that. %The action \textit{take} is Bob's action, so we say that he is the \emph{owner} of the action (to be defined formally below).
  Let us assume that Bob does strong planning, that is, only looks for plans that are guaranteed to reach the goal. The problem is then that, from Bob's perspective, \textit{try-take} can not be guaranteed to reach the goal. This is formally because we have:
\begin{center}
  $s^{\textit{Bob}} \otimes \textit{try-take} =
  \twoindist{$(w,e):h$}
            {$(v,f):$}{}{}{}{}{}{white}{2.5}$
\end{center}
Both worlds being designated, but distinguishable, means that Bob at \emph{plan time} (before executing the action) considers them both as possible outcomes of $\textit{try-take}$, but is aware that he will at \emph{run time} (after having executed the action) know which one it is (see~\cite{bolander11} for a more full discussion of the plan time/run time distinction). Since the world $(v,f)$ is designated, we have $s^{\textit{Bob}} \otimes \textit{try-take} \not\models h$.
So from Bob's perspective, \textit{try-take} might fail to produce $h$ and is hence not a strong plan to achieve $h$. Intuitively, it is of course simply because he does not, at plan time, know whether the key is under the mat or not.

Since $s^{\textit{Anne}} = s$ and $s \otimes \textit{try-take} \models h$, it might seem that \textit{try-take} is still a strong plan to achieve $h$ from the perspective of Anne. But in fact, it is not, at least not of the implicitly coordinated type we will define below. The issue is, \textit{try-take} is still an action that Bob has to execute, but Anne knows that Bob does not know it to be succesful, and she can therefore not expect him to execute it. The idea is that when Anne comes up with a plan that involves actions of Bob, she should change her perspective to his, in order to find out what he can be expected to do. Technically speaking, it is because $(s^{\textit{Anne}})^{\textit{Bob}} \otimes \textit{try-take} \not\models h$ that the plan is not implicitly coordinated from the perspective of Anne.

 %  It is a private sensing action for agent 1, where (only) agent 1gets to know the truth value of $p$, since $e_1$ and $e_2$ are distinguishable to agent 1. Letting $(\model,\{w_1\})$ denote the state from Example~\ref{exam:states}, we get:
%  \begin{center}
%  $(\model,\{w_1 \}) \otimes (\myevent,\{e_1,e_2\}) =
 % \cfig{\scalebox{0.9}{
 %     \begin{tikzpicture}
 %       \path
%	(0,0) node[udworld](a2d){}
%	(0,0) node[uworld, label=below:{$(w_1,e_1):p$}](w1){}
%	(2.5,0) node[uworld,label=below:{$(w_2,e_2):$}](w2){};
   %     \draw (w1) -- node[above] {$2$} (w2);
   %   \end{tikzpicture}
 % }}$
 % \end{center}
  %
 % Hence $(\model,\{w_1 \}) \otimes (\myevent,\{e_1,e_2\})$ is exactly as $(\model,\{w_1\})$ except the indistinguishability edge for agent 1  is removed. So the private sensing action reveals to agent 1 that $p$ is true (without revealing it to agent 2). Before executing the action, agent 1 however does not know whether he will learn $p$ or $\neg p$, which is signified by $(\model, \{w_1 \})^1 \otimes (\myevent, \{e_1,e_2 \})$ having both a designated $p$ world and a  designated $\neg p$ world (it differs from the model $(\model,\{w_1 \}) \otimes (\myevent,\{e_1,e_2\})$ shown above exactly by $(w_2,e_2)$ also being designated).
\end{example}

We extend the language $\lang$ into the \emph{dynamic language}
$\langd$ by adding a modality $[(\myevent,e)]$ for each global action $(\myevent,e)$.
The truth conditions are extended with the following
standard clause from DEL:
$
(\model,w) \models\ [(\myevent,e)] \phi ~\myiff~
(\model,w) \models \pre(e) \text{ implies } (\model,w) \otimes (\myevent,e) \models \phi
$.

We define the following abbreviations: $$[(\myevent,E_d)]\phi :=
\textstyle\bigwedge_{e \in E_d} [(\myevent,e)] \phi \quad \text{ and } \quad \langle
(\myevent, E_d) \rangle\phi := \neg [(\myevent,E_d)] \neg \phi.$$ We say that
an action $(\myevent,E_d)$ is \emph{applicable} in a state
$(\model,W_d)$ if for all $w \in W_d$ there is an event $e\in E_d$
s.t. $(\model,w) \models \pre(e)$. Intuitively, an action is
applicable in a state if for each possible situation (designated
world), at least one possible outcome (designated event) is
specified.
\begin{example}\label{exam:applicability} \upshape
Consider again the state $s$ from Example~\ref{exam:states} and the action $\textit{try-take}$ from Example~\ref{exam:take}. The action \textit{try-take} is trivially seen to be applicable in the state $s$, since the designated event $e$ has its precondition satisfied in the designated world $w$. The action \textit{try-take} is also applicable in $s^\textit{Bob}$, since 
%For each of the worlds $w$ and $v$ that Bob considers possible (the designated worlds), \textit{try-take} specifies a corresponding event that Bob considers possible (a designated event) and that has its precondition satisfied in the world. 
%, as can be seen as follows. In $s^\textit{Bob}$, Bob considers both $w$ and $v$ as possible current situations (they are both designated). So for each of these worlds, there has to be a corresponding designated event in \textit{try-take} having its precondition satisfied in the world. This is satisfied since 
$(\model,w) \models pre(e)$ and $(\model,v) \models pre(f)$. This shows that \textit{try-take} is applicable from the perspective of Bob. Intuitively, it is so because it is only an action for \emph{attempting} to take the key. Even if the key is not under the mat, the action will specify an outcome (having the trivial postcondition $\top$).

\iffalse
Consider instead an action $\textit{take}$ only consisting of the event $e: \langle m, h \land \neg m \rangle$. This would be an action that unconditionally sets $h$ true, but it would only be applicable in states where all designated worlds satisfy $m$. This action is still applicable in $s$, but not in $s^{\textit{Bob}}$: Bob can not verify that he can execute the action `take the key from under the mat', since he does not know the key to be under the mat. Note that this does not contradict the conclusion of the previous paragraph, that Bob \emph{can} verify to be able to execute the action
%However, as shown in the previous paragraph, he can still verify to be able to
%But this of course does not contradict that he can still verify to be able to
%execute the action
`attempt to take the key from under the mat' (\textit{try-take}). He knows he can execute the latter, and that it will either lead to having the key (the world $(w,e)$ of $s^{\textit{Bob}} \otimes \textit{try-take}$, see Example~\ref{exam:take}) or knowing that it is not there (the world $(v,f)$ of $s^{\textit{Bob}} \otimes \textit{try-take}$).
%the latter will either lead to having the key
\fi
\end{example}

Let $s = (\model,W_d)$ denote an epistemic state and $a
=(\myevent,E_d)$ an action. Andersen~\cite{andersentowards} shows that
$a$ is applicable in $s$ iff $s \models \langle a \rangle \top$, and
that $s \models [ a ] \phi$ iff $s \otimes a \models \phi$. We now
define a further abbreviation: $\after{a} \phi := \langle a \rangle
\top \wedge [a] \phi$. Hence:
\begin{equation} \label{dynmod}
  s \models \after{a} \phi ~\myiff~ a \textnormal{ is applicable in } s \textnormal{ and } s \otimes a \models \phi
\end{equation}
Thus $\after{a} \phi$ means that the application of $a$ is possible
and will (necessarily) lead to a state fulfilling $\phi$.
  % Theoretical Background
% !TEX root =  paper.tex

\section{Cooperative Planning}

\enlargethispage{\baselineskip}
%We will now consider different types of planning tasks and solution concepts employing the epistemic states and actions introduced above.
As mentioned in the introduction, in this paper we assume each action to be executable by a single agent, that is, we are not considering joint actions. In our `apartment borrowing' example there are two agents, Anne and Bob. They are supposed to execute different parts of the plan to reach the goal of Bob getting access to the apartment. For instance, Anne is the one to put the key under the mat, and Bob is the one to take it from under the mat. To represent who performs which actions of a plan, we will introduce what we call an \emph{owner function} (inspired by the approach of L\"owe, Pacuit, and Witzel~\cite{lowe11}). An \emph{owner function} is defined to be a mapping $\omega: A \to \agents$, mapping each action
to the agent who can execute it. For any action $a \in A$, we call $\omega(a)$ the \emph{owner} of $a$. Note that by defining the owner function this way, every action has a \emph{unique} owner. This can be done without loss of generality, since we can always include any number of semantically equivalent actions in $A$, if we wish some action to be executable by several agents (e.g., if we want both Anne and Bob to be able to open and close the door).
We can now define epistemic planning tasks.

\leaveout{
In automated planning, a planning task usually consists of an initial state $s_0$, a set of available actions, $A$, and a goal formula $\gamma$. In our setting, all states and actions will be of the epistemic type introduced above, and the goal formula will be a formula of the epistemic language. Furthermore, we need each planning task to specify the relevant owner function. Hence the definition of a planning task becomes as follows.}

\begin{definition}
  A \emph{cooperative planning task} (or simply a \emph{planning
    task}) $\Pi = \tpl s_0,A,\owner,\gamma \tpr$ consists of
  %\begin{itemize}
  %\item
  an initial epistemic state $s_0$,
  %\item
  a finite set of epistemic actions $A$,
  %\item
  an owner function $\owner : A \to \agents$, and
  %\item
  a goal formula $\gamma$  of $\lang$.
  %\end{itemize}
  Each $a \in A$ has to be local for $\owner(a)$. When $s_0$ is a
  global state, we call it a \emph{global planning task}. When $s_0$ is
  local for agent $i$, we call it a \emph{planning task
    for agent $i$}. Given a planning task $\Pi = \langle s_0, A, \owner, \gamma \rangle$ and a state $s$, we define $\Pi(s) := \langle s , A, \owner, \gamma \rangle$.
     Given a planning task $\Pi = \tpl
  s_0,A,\owner,\gamma \tpr$, the \emph{associated planning task for
    agent $i$} is $\Pi^i = \Pi(s^i)$.
\end{definition}

Given a multi-agent system facing a global planning task $\Pi$,
each individual agent $i$ is facing the planning task $\Pi^i$
(agent $i$ cannot observe the global initial state $s_0$ directly, only the
associated local state $\als{s_0}{i}$).

In the following, we will investigate various possible solution concepts for cooperative planning tasks. The solution to a planning task is called a \emph{plan}. A plan can either be sequential (a sequence of actions) or conditional (a policy). We will first, in Section~\ref{sect:standardsequential}, consider the simplest possible type of solution, a \emph{standard sequential plan}. Then, in Section~\ref{sect:sequentialplans}, we are going to introduce the more complex notion of a \emph{sequential implicitly coordinated plan}, and in Section~\ref{sect:conditionalplans} this will be generalized to \emph{implicitly coordinated policies}.

\subsection{Standard Sequential Plans}\label{sect:standardsequential}

The standard notion of a sequential plan in automated planning is as follows (see, e.g., \cite{ghallab04}). An action sequence $(a_1,\dots,a_n)$ is called a (standard sequential) plan if for every $i \leq n$, $a_i$ is applicable in the result of executing the action sequence $(a_1,\dots,a_{i-1})$ in the initial state, and when executing the entire sequence $(a_1,\dots,a_n)$ in the initial state, a state satisfying the goal formula is reached. Let us transfer this notion into the DEL-based setting of this paper. %Let $\Pi = \tpl s_0,A,\owner,\gamma \tpr$ be a cooperative planning task.
In our setting, the result of executing an action $a$ in a state $s$ is given as $s \otimes a$. Hence, the above conditions for $(a_1,\dots,a_n)$ being a plan can be expressed in the following way in our setting, where $s_0$ denotes the initial state, and $\gamma$ the goal formula: for every $i \leq n$, $a_i$ is applicable in $s_0 \otimes a_1 \otimes \cdots \otimes a_{i-1}$, and $s_0 \otimes a_1 \otimes \cdots \otimes a_n \models \gamma$. Note that by equation (\ref{dynmod}) above, these conditions are equivalent to simply requiring $s_0 \models \after{a_1} \after{a_2} \cdots \after{a_n} \gamma$. Hence we get the following definition.

%In the DEL-based setting considered here, the result of executing an action $a$ in a state $s$ is given as $s \otimes a$. Hence, the

% In the DEL-based setting, the state-transition function mapping a state-action pair $(s,a)$ into the state resulting from executing $a$ in $s$ is given by $(s,a) \mapsto s \otimes a$ (when $a$ is not applicable in $s$, the state-transition function is taken to be undefined on $(s,a)$).  Hence, more formally, a solution to $\langle s_0, A, \gamma \rangle$ is a sequence of actions $(a_1,\dots,a_n)$ from $A$ such that for all $i=1,\dots,n$, the action $a_i$ is applicable in $s_0 \otimes a_1 \otimes \cdots \otimes a_{i-1}$, and $s_0 \otimes a_1 \otimes \cdots \otimes a_n \models \gamma$. Note that by (\ref{dynmod}) above, these conditions are equivalent to simply requiring $s_0 \models \after{a_1} \after{a_2} \cdots \after{a_n} \gamma$.

%Let $A$ denote a set of actions, and $\agents$ a set of agents. In this paper we assume each action to be executable by a single agent, that is, we are not considering joint actions. For technical reasons, we wish each action to be executable by a \emph{unique} agent, which we call the \emph{owner} of the action. More precisely, an \emph{owner function} is a mapping $\omega: A \to \agents$, mapping each action to the agent who can execute it (who \emph{owns} it). This approach is closely related to the one by L\"owe, Pacuit, and Witzel~\cite{lowe11}. Mapping each action to a unique agent can be done without loss of generality, since semantically equivalent duplicates can always be added to the action set.

\begin{definition}\label{defi:standardsequential}
Let $\Pi =  \tpl s_0,A,\owner,\gamma \tpr$ be a planning task. A \emph{standard sequential plan} for $\Pi$ is a sequence $(a_1,\dots,a_n)$ of actions from $A$ satisfying
$
  s_0 \models \after{a_1} \after{a_2} \cdots \after{a_n} \gamma
$.
\end{definition}

%In the simplest case, a planning task $\langle s_0, A, \gamma \rangle$ consists of an initial (epistemic) state $s_0$, a set $A$ of (epistemic) actions and a goal formula $\gamma$ of $\lang$. Informally, a (sequential) solution to such a planning task is a sequence of actions $(a_1,\dots,a_n)$ from $A$, such that executing the sequence in $s_0$ leads to a state satisfying $\gamma$.

% In the DEL-based setting, the state-transition function mapping a state-action pair $(s,a)$ into the state resulting from executing $a$ in $s$ is given by $(s,a) \mapsto s \otimes a$ (when $a$ is not applicable in $s$, the state-transition function is taken to be undefined on $(s,a)$).  Hence, more formally, a solution to $\langle s_0, A, \gamma \rangle$ is a sequence of actions $(a_1,\dots,a_n)$ from $A$ such that for all $i=1,\dots,n$, the action $a_i$ is applicable in $s_0 \otimes a_1 \otimes \cdots \otimes a_{i-1}$, and $s_0 \otimes a_1 \otimes \cdots \otimes a_n \models \gamma$. Note that by (\ref{dynmod}) above, these conditions are equivalent to simply requiring $s_0 \models \after{a_1} \after{a_2} \cdots \after{a_n} \gamma$.

This solution concept is equivalent to the one considered
in~\cite{bolander11}. As the following example shows, it is not sufficiently strong to capture the notion of an `implicitly coordinated plan' that we are after. % in this paper.

% In that paper, the initial state as well as all actions are supposed to be modelled from the perspective of one single planning agent, that is, be local to that agent. Such a setting provides a natural formal framework for a single agent acting alone in a multi-agent environment, but does not provide a systematic solution to the case where multiple agents are (inter)acting towards a joint goal. The latter situation is what we wish to consider in this paper, and we hence now move to consider \emph{implicitly coordinated plans}.

\enlargethispage{\baselineskip}
\begin{example}\label{exam:standardsequential} \upshape
Consider again the scenario of Example~\ref{exam:take} where the key is initially under the mat, Bob does not know this, and the goal is for Bob to have the key. The only action available in the scenario is for Bob to attempt to take the key from under the mat. Using the states and actions defined in Examples~\ref{exam:states} and \ref{exam:take}, we can express this scenario as a cooperative planning task $\Pi = \langle s_0, A, \owner, \gamma \rangle$ where $s_0 = s$, $A = \{ \textit{try-take} \}$, $\owner(\textit{try-take}) = \textit{Bob}$, $\gamma = h$. In Example~\ref{exam:take}, we informally concluded that the plan only consisting of \textit{try-take} is a strong plan, since it is guaranteed to reach the goal, but that it is not a strong plan from the perspective of Bob. Given our formal definitions, we can now make this precise as follows:
\begin{enumerate}
  \item The sequence $(\textit{try-take})$ is a standard sequential plan for $\Pi$.
  \item The sequence $(\textit{try-take})$ is \emph{not} a standard sequential plan for $\Pi^\textit{Bob}$.
\end{enumerate}
The first item follows from the fact that \textit{try-take} is applicable in $s$ (Example~\ref{exam:applicability}), and that $s \otimes \textit{try-take} \models h$ (Example~\ref{exam:take}). The second item follows from $s^\textit{Bob} \otimes \textit{try-take} \not\models h$ (Example~\ref{exam:take}).

We also have that $(\textit{try-take})$ is a standard sequential plan for $\Pi^{\textit{Anne}}$, since $s^\textit{Anne} = s$. This proves that the notion of a standard sequential plan is insufficient for our purposes. If Anne is faced with the planning task $\Pi^\textit{Anne}$, she should not be allowed to consider $(\textit{try-take})$ as a sufficient solution to the problem. She should be aware that the action \textit{try-take} is to be executed by Bob, and from Bob's perspective, $(\textit{try-take})$ is \emph{not} a (strong) solution to the planning task (Item 2 above). So we need a way of explicitly integrating perspective-shifts into our notion of a solution to a planning task, and this is what we will do next.

%In Example~\ref{exam:take} we concluded that $s \otimes \textit{try-take} \models h$. In Example~\ref{exam:applicability}, we concluded that \textit{try-take} is applicable in $s$. Hence we immediately get that $s_0 \models \after{\textit{try-take}} \gamma$, or in other words, that $(\textit{try-take})$ is a standard sequential plan for $\Pi$.

%The action sequence $(\textit{try-take})$ qualifies as a plan since it \emph{will} be successful. However, as noted in Example~\ref{exam:take}, the plan can not be verified from the perspective of Bob. Formally, this is because, as we showed in Example~\ref{exam:take}, $s^\textit{Bob} \otimes \textit{try-take} \not\models h$. From $s^\textit{Bob} \otimes \textit{try-take} \not\models h$ we immediately get $s^\textit{Bob} \not\models \after{\textit{try-take}} h$, showing that
%This immediately shows that
% $(\textit{try-take})$ is not a standard sequential plan for Bob's associated planning task $\Pi^\textit{Bob}$.

% we see that by $(\textit{try-take})$ not being a standard sequential plan for Bob's associated planning task $\Pi^\textit{Bob}$: $s_0^\textit{Bob} \not\models \after{\textit{try-take}} \gamma$.

\end{example}

\subsection{Implicitly Coordinated Sequential Plans}\label{sect:sequentialplans}
It follows from Definition~\ref{defi:standardsequential} that $(a_1,\dots,a_n)$ is a standard sequential plan for some planning task $\Pi = \langle s_0,A,\owner,\gamma \rangle$ iff $a_1,\dots,a_n \in A$ and
the formula $\after{a_1} \cdots \after{a_n} \gamma$ is true in $s_0$. %The formula $\after{a_1} \cdots \after{a_n} \gamma$ gives a logical description of what it means to be a standard sequential plan (for a planning task with goal $\gamma$).
More generally, we can think of a planning notion as being
defined via a mapping $X$ that takes a plan $\pi$ and a planning task $\Pi$ as parameters, and returns a logical formula $X(\pi,\Pi)$ such that, for all states $s$,  $s \models X(\pi,\Pi)$ iff $\pi$ is a plan for $\Pi(s)$.
%describing that $\pi$ is a plan for $\Pi$.
In the case of standard sequential plans, it follows directly from Definition~\ref{defi:standardsequential} that $X$ would be defined like this: %, for all $\Pi = \langle s_0,A, \owner, \gamma \rangle$ and  $a_1,\dots,a_n \in A$:
% cf.\ Definition~\ref{defi:standardsequential}:
\begin{equation}
%  \begin{array}{l}
 % \text{If $\Pi = \langle s_0,A, \owner, \gamma \rangle$ and  $a_1,\dots,a_n \in A$, then} \\
   \text{$X((a_1,\dots,a_n),\langle s_0,A,\owner,\gamma \rangle) = \after{a_1} \cdots \after{a_n} \gamma$}
%  \end{array}
\end{equation}
for all planning tasks $\langle s_0,A,\owner,\gamma \rangle$ and all $a_1,\dots,a_n \in A$.
%The essential property satisfied by $X$ is the following: $s \models X(\pi,\Pi)$ iff $\pi$ is a standard sequential plan for $\Pi(s)$.

%Hence the formula $\after{a_1} \cdots \after{a_n} \gamma$ is the \emph{logical condition} describing when $(a_1,\dots,a_n)$ is a standard sequential plan. More formally, we can thin
We now wish to define a similar mapping $Y$, so that $s \models Y(\pi,\Pi)$ iff $\pi$ is an \emph{implicitly coordinated plan} for $\Pi(s)$. Our strategy is to list the natural conditions that $Y$ should satisfy, and then derive the exact definition of $Y$ (and hence implicitly coordinated plans) directly from those.
% We wish to find a similar logical condition for when a sequence $(a_1,\dots,a_n)$ is an \emph{implicitly coordinated plan}. More precisely, we are looking for a mapping $X$ from action sequences into formulas of $\langd$, such that $X(a_1,\dots,a_n)$ is a logical description of what it takes for $(a_1,\dots,a_n)$ to be an implicitly coordinated plan. We will try to deduce the correct definition of $X$ from the natural properties it should satisfy.
First of all, the empty action sequence, denoted by $\epsilon$, should be an implicitly coordinated plan iff it satisfies the goal formula, which is expressed by the following simple condition on $Y$:
\begin{equation}\label{equa:basis}
   Y(\epsilon, \langle s_0, A, \owner, \gamma \rangle) = \gamma.
 \end{equation}

For non-empty action sequences, the `apartment borrowing' example studied above gives us the following insights. If Anne is trying to come up with a plan where one of the steps is to be executed by Bob, then Anne has to make sure that Bob can himself verify his action to be applicable, and that he can himself verify that executing the action will lead to a state where the rest of the plan will succeed. More generally, for an action sequence $(a_1,\dots,a_n)$ to be considered implicitly coordinated, the owner of the first action $a_1$ has to \emph{know} that $a_1$ is applicable and will lead to a situation where $(a_2,\dots,a_n)$ is again an implicitly coordinated plan. This leads us directly to the following condition on $Y$, for all planning tasks $\langle s_0,A,\owner,\gamma \rangle$ and all $a_1,\dots,a_n \in A$ with $n \geq 1$:
\begin{equation}\label{equa:step}
\begin{array}{l}
Y((a_1,\dots,a_n), \langle s_0, A, \owner, \gamma \rangle) =
   K_{\owner(a_1)} \after{a_1} Y((a_2,\dots,a_n), \langle s_0, A, \owner, \gamma \rangle)
   \end{array}
%  \models X(a_1,\dots,a_n) \leftrightarrow K_{\owner(a_1)} \after{a_1} X(a_2,\dots,a_n)
\end{equation}
%Condition (\ref{equa:step}) in words: $(a_1,\dots,a_n)$ is an implicitly coordinated plan iff the owner of action $a_1$ knows that $a_1$ is applicable and will lead to a situation where $(a_2,\dots,a_n)$ is an implicitly coordinated plan.
It is now easy to see that any mapping $Y$ satisfying  (\ref{equa:basis}) and (\ref{equa:step}) must necessarily be defined as follows, for all planning tasks $\langle s_0, A, \owner, \gamma \rangle$ and all action sequences $(a_1,\dots,a_n) \in A$:
 \[
  Y((a_1,\dots,a_n), \langle s_0, A, \owner, \gamma \rangle) =
   \kafter{a_1} \kafter{a_2} \cdots \kafter{a_n} \gamma
%  \! \leftrightarrow\! \kafter{a_1} \kafter{a_2} \cdots \kafter{a_n} \gamma
  \]
\iffalse
It is now trivial to prove the following.
\begin{proposition}\label{prop:X}
  A mapping $Y$ satisfies (\ref{equa:basis}) and (\ref{equa:step}) iff
  \[
  \models\! X(a_1,\dots,a_n)\! \leftrightarrow\! \kafter{a_1} \kafter{a_2} \cdots \kafter{a_n} \gamma
  \]
\end{proposition}
\fi
This leads us directly to the following definition. % of implicitly coordinated plans.
\begin{definition}\label{defi:implicitlycoordinated}
  Let $\Pi = \tpl s_0, A, \owner, \gamma \tpr$ be a cooperative
  planning task. An \emph{implicitly coordinated plan} for $\Pi$ is a sequence $(a_1,\dots,a_n)$ of actions
  from $A$ such that:
  \begin{equation}
    s_0 \models \kafter{a_1} \kafter{a_2} \cdots \kafter{a_n} \gamma \label{rec}
  \end{equation}
  If $(a_1,\dots,a_n)$ is an implicitly coordinated plan for $\Pi^i$, then it is said to be
  %A plan for $\Pi^i$ is called
  an \emph{implicitly coordinated plan for agent $i$} to the planning task $\Pi$.
 % If $\Pi$ is a planning task for agent $i$, we call the plan a \emph{plan for agent $i$}.
  %  A \emph{plan for agent $i$} to a planning task $\Pi$ is a plan for $\Pi^i$.
\end{definition}
Note that the formula used to define implicitly coordinated plans above % of Definition~\ref{defi:implicitlycoordinated}
is \emph{uniquely} determined by the natural conditions (\ref{equa:basis}) and (\ref{equa:step}).

The following proposition gives a more structural, and semantic, characterization of
implicitly coordinated plans. It becomes clear that such plans
can be found by performing a breadth-first search over the set of
successively applicable actions, shifting the perspective for each
state transition to the owner of the respective action.
\begin{proposition} \label{slem}
  For a cooperative planning task $\Pi = \tpl s_0, A, \owner, \gamma
  \tpr$, a non-empty sequence $(a_1,\dots,a_n)$ of actions from $A$  is an implicitly coordinated plan for $\Pi$ iff $a_1$ is applicable in $s_0^{\owner(a_1)}$ and $(a_2,\dots,a_n)$ is an implicitly coordinated plan for $\Pi(s_0^{\owner(a_1)} \otimes a_1)$.
\end{proposition}
The proposition can be seen as a semantic counterpart of (\ref{equa:step}), and is easily proven using (\ref{dynmod}),
Proposition~\ref{lemma:localisation} and (\ref{rec}).

\iffalse
  \label{slem}
  \begin{itemize}
  \item $\epsilon$ is an implicitly coordinated plan for $\Pi$ iff $s_0
    \models \phi$.
  \item $(a_1,\dots,a_n)$ with $n \geq 1$ is an implicitly coordinated
    plan for $\Pi$ iff $a_1$ is applicable in $\als{s_0}{\owner(a_1)}$
    and $(a_2,\dots,a_n)$ is an implicitly coordinated plan for $\tpl
    \als{s_0}{\owner(a_1)} \otimes a_1, A, \owner, \phi \tpr$.\qed
  \end{itemize}
\end{proposition}
The proof is simple and hence omitted (it relies on (\ref{dynmod}),
Lemma~\ref{lemma:localisation} and (\ref{rec})).
\fi

\begin{example} \upshape
  Consider again the planning task $\Pi = \langle s_0, A, \owner, \gamma \rangle$ of Example~\ref{exam:standardsequential}
   with
  $s_0 = s$, $A = \{ \textit{try-take} \}$, $\owner(\textit{try-take}) = \textit{Bob}$, $\gamma = h$. In Example~\ref{exam:standardsequential} we concluded that $(\textit{try-take})$ is a standard sequential plan for $\Pi$. From Example~\ref{exam:take}, we have that $s^\textit{Bob} \otimes \textit{try-take} \not\models h$, and hence $s \not\models K_{\owner(\textit{try-take})} \after{\textit{try-take}} h$ (using (\ref{dynmod}) and Proposition~\ref{lemma:localisation}). This shows that, as expected, (\textit{try-take}) is not an implicitly coordinated plan for $\Pi$.

In the introduction, we noted that the solution to this problem would be for Anne to make sure to announce the location of the key to Bob. Let us now treat this formally within the developed framework. We need to give Anne the possibility of announcing the location of the key, so we define a new planning task $\Pi' = \langle s_0, A', \owner, \gamma \rangle$ with $A' = \{ \textit{try-take}, \textit{announce} \}$. Here \textit{announce} is the action \onestate{}{}{}~$e: \langle m, \top \rangle$
        with $\owner(\textit{announce}) = \textit{Anne}$. In DEL, this action is known as a \emph{public announcement} of $m$.
        It can  now  easily be formally verified that
        \[
          s \models K_{\textit{Anne}} \after{\textit{announce}} K_\textit{Bob} \after{\textit{try-take}} h.
        \]
        In words: Anne knows that she can announce the location of the key, and that this will lead to a situation where Bob knows he can attempt to take the key, and he knows that he will be successful in this attempt. In other words, $(\textit{announce},\textit{try-take})$ is indeed an implicitly coordinated plan to achieve that Bob has the key, consistent with our informal analysis in the introduction of the paper.

\end{example}

\begin{example} \upshape \label{rex}
  Consider a situation with agents
  $\mathcal{A} = \{1,2,3\}$ where a letter is to be passed from agent
  1 to one of the other two agents, possibly via the third agent.
  Mutually exclusive propositions $\at1, \at2, \at3 \in P$ are used to
  denote the current carrier of the letter, while $\for1, \for2, \for3
  \in P$ denote the addressee. In our example, agent $1$ has a letter
  for agent $3$, so $\at1$ and $\for3$ are initially true.
  \begin{center}
  $s_0 = \twoindist{$\at1,\for2$}{$\at1,\for3$}{$2,3$}{}{}{white}{}{}{2}$
  \end{center}
  In $s_0$, all agents know that agent $1$ has the letter ($\at1$),
  but agents $2$ and $3$ do not know who of them is the addressee
  ($\for2$ or $\for3$).  We assume that agent 1 can only exchange
  letters with agent 2 and agent 2 can only exchange letters with
  agent 3. We thus define the four possible actions $a_{12}, a_{21},
  a_{23}, a_{32}$, with $a_{ij}$ being the composite action of agent
  $i$ publicly passing the letter to agent $j$ and privately informing
  him about the correct addressee (the name of the addressee is on the envelope, but only visible to the receiver). I.e.
  \begin{center}
  $ a_{ij} = \twoindist
{$\left<\at i \land \for 2, \neg \at i \land \at j \right>$}
{$\left<\at i \land \for 3, \neg \at i \land \at j \right>$}
{$\agents \setminus \left\{j\right\}$}{}{}{}{}{}{4} $
  \end{center}
  Given that the joint goal is to pass a letter to its addressee, the
  global planning task then is $\Pi = \left<s_0, A, \owner,
  \gamma \right>$ with
  %\begin{itemize}
  %\item
  $A = \left\{ a_{12}, a_{21}, a_{23}, a_{32} \right\}$,
  %\item
  $\owner(a_{ij})  = i$ for all $i,j$, and
  %\item
  $\gamma = \bigwedge_{i \in \myset{1,2,3}}\ (\for i \to \at i)$.
  %\end{itemize}
%
  Consider the action sequence $(a_{12},a_{23})$: Agent 1 passes the
  letter to agent 2, and agent 2 passes it on to agent 3. It can now
  be verified that
  %\begin{align*}
    $\als{s_0}{1} \models K_1 \after{a_{12}} K_2 \after{a_{23}} \gamma$ and % \\
    $\als{s_0}{i} \not\models K_1 \after{a_{12}} K_2 \after{a_{23}} \gamma$ for $i=2,3$.
  %\end{align*}
  Hence $(a_{12},a_{23})$ is an implicitly coordinated plan for
  agent $1$, but not for agents $2$ and $3$.
  This is because in the beginning agents $2$ and $3$ do not know to
  whom of them the letter is intended and hence cannot verify that
  $(a_{12}, a_{23})$ will lead to a goal state. However, after agent
  $1$'s execution of $a_{12}$, agent $2$ can distinguish between the
  possible addressees at run time, and find his subplan $(a_{23})$, as
  contemplated by agent $1$.
\end{example}

%@@@: add: one level of ToM is not enough. Not even any finite level.
  % Cooperative Epistemic Planning - Sequential
% !TEX root =  paper.tex

\subsection{Conditional Plans}\label{sect:conditionalplans}

\emph{Sequential} plans are often not sufficient to solve a given epistemic
planning task. In particular, as soon as branching on nondeterministic
action outcomes or obtained observations becomes necessary, we need
\emph{conditional} plans to solve such a task.  
%Consider for instance a task where agents 1 and 2 need to cooperate as follows: Agent 1 starts out
%with an action $a_1$ that provides necessary information to agent 2 on how to
%continue, and depending on the new information, agent 2 needs to continue either
%with action $a_2$ or $a_3$ in order to achieve the joint goal.
%
Unlike Andersen, Bolander, and Jensen~\cite{andersen12}, who represent
conditional plans as action trees with branches depending on knowledge
formula conditions, we represent them as policy functions $(\pi_i)_{i
  \in \agents}$, where each $\pi_i$ maps minimal local states of agent $i$ into actions of agent $i$.
  % mapping minimal local epistemic states to actions for their respective observer agents.
We now define two different types of policies, \emph{joint policies}
and \emph{global policies}, and later show them to be equivalent.

\begin{definition}[Joint policy] \label{defpolicy}
  Let $\Pi = \tpl s_0, A, \owner, \gamma \tpr$ be a cooperative
  planning task. Then a \emph{joint policy} $(\pi_i)_{i \in
    \agents}$ consists of partial functions $\pi_i : \smin{i} \to A$
  satisfying for all states $s$ and actions $a$: if $\pi_i(s) = a$
 %$(s,a) \in \pi_i$,
 then
  $\owner(a) = i$ and $a$ is
  applicable in $s$.
\end{definition}

\begin{definition}[Global policy] \label{defglobalpolicy}
  Let $\Pi = \tpl s_0, A, \owner, \gamma \tpr$ be a cooperative
  planning task. Then a \emph{global policy} $\pi$ is a mapping
  $\pi: \sglobal \to \mathcal{P}(A)$ satisfying the requirements
  \emph{knowledge of preconditions} \textsc{(kop)},
  \emph{per-agent determinism} \textsc{(det)}, and
  \emph{uniformity} \textsc{(unif)}:
  \begin{itemize}
  \item[\textsc{(kop)}] For all $s \in \sglobal$, $a \in \pi(s)$: $a$ is applicable in
    $\als{s}{\owner(a)}$.
  \item[\textsc{(det)}] For all $s \in \sglobal, a, a' \in \pi(s)$ with
      $\owner(a) = \owner(a')$: $a = a'$.
  \item[\textsc{(unif)}] For all $s, t \in \sglobal, a \in \pi(s)$
    with $\als{s}{\owner(a)} = \als{t}{\owner(a)}$: $a \in \pi(t)$.
  \end{itemize}
\end{definition}

\begin{proposition}\label{policyreq}
  Any joint policy $(\pi_i)_{i \in \agents}$ induces a global policy
  $\pi$ given by
  \[
  \pi(s) = \mysetw{\pi_i(\als{s}{i})}{i \in \agents \text{ and } \pi_i(\als{s}{i}) \text{ is defined}}.
  \]
  Conversely, any global policy $\pi$ induces a joint policy
  $(\pi_i)_{i\in \agents}$ given by
  \[
  \pi_i (s^i ) = a ~~ \text{ for all  } (s,A') \in \pi, ~a \in A' \text{ with } \owner(a) = i.
  \]
  Furthermore, the two mappings $(\pi_i)_{i \in \agents} \mapsto \pi$ (mapping joint policies to induced global policies) and $\pi \mapsto (\pi_i)_{i \in \agents}$ (mapping global policies to induced joint policies) are each other's inverse.
\end{proposition}
\enlargethispage{1.5\baselineskip}
\begin{proof}
  First we prove that the induced mapping $\pi$ as defined above is a
  global policy. Condition \textsc{(kop)}: If $a \in \pi(s)$ then $\pi_i(s^i) =
  a$ for some $i$, and by definition of joint policy this implies $a$
  is applicable in $s^i$. Condition \textsc{(det)}:
  Assume $a,a' \in \pi(s)$ with $\owner(a) =
  \owner(a')$. By definition of $\pi$ we have $\pi_i(s^i) = a$ and
  $\pi_j(s^j) = a'$ for some $i,j$. By definition of joint policy,
  $\owner(a) = i$ and $\owner(a') = j$. Since $\owner(a)= \owner(a')$
  we get $i=j$ and hence $\pi_i(s^i) = \pi_j(s^j)$. This implies $a =
  a'$. Condition \textsc{(unif)}: Assume $a\in \pi(s)$ and $s^{\owner(a)} =
  t^{\owner(a)}$. By definition of $\pi$ and joint policy, we get
  $\pi_i (s^i ) = a$ for $i = \owner(a)$. Thus $s^i = t^i$, and since
  $\pi_i (s^i) = a$, we immediately get $\pi_i (t^i) = a$ and hence $a
  \in \pi(t)$.
  We now prove that the induced mappings $(\pi_i)_{i \in \agents}$
  defined above form a joint policy. Constraint \textsc{(kop)} ensures the
  applicability property as required by Definition \ref{defpolicy},
  while the constraints \textsc{(det)} and \textsc{(unif)} ensure the
  right-uniqueness of each partial function $\pi_i$.
  It is easy to show that the two mappings $(\pi_i)_{i \in \agents} \mapsto \pi$ and $\pi \mapsto (\pi_i)_{i \in \agents}$ are each other's inverse, using their definition.
  % TODO: proof pi_i -> pi -> pi_i
  %             pi -> pi_i -> pi
\end{proof}

By Proposition~\ref{policyreq}, we can identify joint and global
policies, and will in the following move back and forth between the
two. Notice that Definitions~\ref{defpolicy}~and~\ref{defglobalpolicy}
allow a policy to distinguish between modally equivalent states. A
more sophisticated definition avoiding this is possible, but is beyond
the scope of this paper.
Usually, a policy $\pi$ is only considered to be a solution to a
planning task if it is closed in the sense that $\pi$ is defined for
all non-goal states reachable following $\pi$. Here, we want to
distinguish between two different notions of closedness: one that
refers to all states reachable from a centralized perspective, and one
that refers to all states considered reachable when tracking
perspective shifts. To that end, we distinguish between centralized
and perspective-sensitive successor functions.

%\begin{definition}
  We take a  \emph{successor function} to be any function $\succ : \sglobal \times A \to
  \mathcal{P}(\sglobal)$. Successor functions are intended to map pairs of states $s$ and actions $a$ into the states $\succ(s,a)$ that can result from executing $a$ in $s$. Which states can result from executing $a$ in $s$ depend on whether we take the objective, centralized view, or whether we take the subjective view of an agent. An agent might subjectively consider more outcomes possible than are objectively possible. We define the \emph{centralized successor function} as
$\succ_{cen}(s,a) = \globals(s \otimes a)$.
It specifies the global states that are possible after the application
of $a$ in $s$. If closedness of a global policy $\pi$ based on the
centralized successor function is required, then no execution of $\pi$
will ever lead to a non-goal state where $\pi$ is undefined.
Like for sequential planning, we are again interested in the
decentralized scenario where each agent has to plan and decide when
and how to act by himself under incomplete knowledge. We achieve this
by encoding the perspective shifts to the next agent to act in the
\emph{perspective-sensitive successor function}
$\succ_{ps}(s,a) = \globals(\als{s}{\owner(a)} \otimes a)$.
Unlike $\succ_{cen}(s,a)$, $\succ_{ps}(s,a)$ considers a state
$s'$ to be a successor of $s$ after application of $a$ if \emph{agent
  $\owner(a)$ considers $s'$ possible} after the application of $a$,
not only if \emph{$s'$ is actually possible from a global perspective}. Thus,
$\succ_{cen}(s,a)$ is always a (possibly strict) subset of
$\succ_{ps}(s,a)$, and a policy $\pi_{ps}$ that is closed
wrt.\ $\succ_{ps}(s,a)$ must be defined for \emph{at least} the
states for which a policy $\pi_{cen}$ that is closed
wrt.\ $\succ_{cen}(s,a)$ must be defined. This corresponds to the
intuition that solution existence for decentralized planning with
implicit coordination is a stronger property than solution existence
for centralized planning.
For both successor functions, we can now formalize what a strong
solution is that can be executed by the agents. Our
notion satisfies the usual properties of strong plans
\cite{cimatti03weak}, namely closedness, properness and acyclicity.

%\enlargethispage{\baselineskip}
\begin{definition}[Strong Policy]
  Let $\Pi = \tpl s_0,A,\owner,\gamma \tpr$ be a cooperative planning
  task and $\succ$ a successor function. A global policy $\pi$ is
  called a \emph{strong policy} for $\Pi$ with respect to $\succ$ if
  \begin{compactenum}[(i)]
  \item \emph{Finiteness:} $\pi$ is finite.
  \item \emph{Foundedness:} for all $s \in \globals(s_0)$, %\\
    \begin{inparaenum}[(1)]
    \item $s \models \gamma$, or
    \item $\pi(s) \neq \emptyset$.
\end{inparaenum}
  \item \emph{Closedness:} for all $(s,A') \in \pi$, $a \in A', s' \in \succ(s,a)$, %\\
    \begin{inparaenum}[(1)]
    \item $s' \models \gamma$, or
    \item $\pi(s') \neq \emptyset$.
    \end{inparaenum}
  \end{compactenum}
  \label{def:strongpolicy}
\end{definition}

Note that we do not explicitly require acyclicity, since this is
already implied by a literal interpretation of the product update
semantic that ensures unique new world names after each update. It
then follows from (i) and (iii) that $\pi$ is proper. We call strong
plans with respect to $\succ_{cen}$ {\em centralized policies} and
strong plans with respect to $\succ_{ps}$ {\em implicitly coordinated
policies}.

% Implicitly coordinated policies come with the guarantee that in each state
% that can be reached following the policy, at least one agent $i$ is aware of
% some action that agent $i$ itself can execute and that advances the state
% towards the goal. More formally:

\enlargethispage{1.5\baselineskip}
Analogous to Proposition~\ref{slem}, we want to give a semantic characterization
of implicitly coordinated policies. For this, we first define a \emph{successor of a
state $s_0$ by following a policy $\pi$} to be a state $s \in \globals(s'_0
\otimes a)$ for arbitrary states $s'_0 \in \globals(s_0)$ and arbitrary actions
$a \in \pi(s'_0)$. We can then show that if $\pi$ is an implicitly coordinated
policy for a planning task $\Pi$, each successor state $s$ of the initial state
$s_0$ either will already be a goal state, or there will be some agent $i \in
\agents$ who can find an implicitly coordinated policy for his own associated planning task
prescribing an action for himself.

\begin{proposition}\label{prop:strong}
  Let $\pi$ be an implicitly coordinated policy for a planning task $\Pi = \tpl
  s_0, A, \owner, \gamma \tpr$ and let $s$ be a non-goal successor state of $s_0$
  by following $\pi$. Then there is an agent $i \in A$ such that $\pi(s)$
  contains at least one of agent $i$'s actions and $\pi$ is an implicitly
  coordinated policy of $\Pi(s^i)$.
\end{proposition}

\begin{proof}
  The existence of an action $a \in \pi(s)$ with \emph{some} owner $i$ follows
  directly from the closedness of implicitly coordinated policies. We need to
  show that $\pi$ is also implicitly coordinated for $\Pi(s^i)$. Finiteness and
  closedness of $\pi$ still hold for $\Pi(s^i)$, since $\pi$ was already finite
  and closed for $\Pi$, and this does not change when replacing $s_0$ with
  $s^i$. For foundedness of $\pi$ for $\Pi(s^i)$, we have to show that $\pi$ is
  defined and returns a nonempty set of actions for all global states $s' \in
  \globals(s^i)$. For $s$ itself, we already known that $a \in \pi(s)$. By
  uniformity, since all such $s'$ are indistinguishable from $s$ for agent $i$,
  $\pi$ must assign the same action $a$ to all states $s' \in \globals(s^i)$.
\end{proof}

\begin{example}\label{letter2} \upshape
  Consider again the letter passing problem introduced in
  Example \ref{rex}. Let $s_{0,2}$ and $s_{0,3}$ denote the global
  states that are initially considered possible by agent $2$.
  \begin{center}
    $s_{0,2} = \twoindist{$\at 1,\for 2$}{$\at 1,\for 3$}{$2,3$}{}{}{}{white}{}{2}$ \hspace{2cm}
    $s_{0,3} = \twoindist{$\at 1,\for 2$}{$\at 1,\for 3$}{$2,3$}{}{}{white}{}{}{2}$
  \end{center}
  With $s_{1,3} = s_{0,3} \otimes a_{12}$, a policy for agent $2$ is
  given by
  $\pi_1 = \myset{s_{0,3} \mapsto a_{12}, s_{0,2} \mapsto a_{12}},
  \pi_2 = \myset{s_{1,3} \mapsto a_{23}}.$ After the contemplated
  application of $a_{12}$ by agent $1$ (in both cases), agent $2$ can
  distinguish between $s_{1,2} = s_{0,2} \otimes a_{12}$, where the
  goal is already reached and nothing has to be done, and $s_{1,3}$,
  where agent $2$ can apply $a_{23}$, leading directly to the goal
  state $s_{1,3} \otimes a_{23}$. Thus, $\pi$ is an implicitly
  coordinated policy for $\als{\Pi}{2}$.
  While in the sequential case, agent $2$ has to wait for the first
  action $a_{12}$ of agent $1$ to be able to find its subplan, it can
  find the policy $(\pi_i)_{i \in \agents}$ in advance by explicitly
  planning for a run-time distinction.
\end{example}

In general, strong policies can be found by performing an AND-OR
search, where AND branching corresponds to branching over different
epistemic worlds and OR branching corresponds to branching over
different actions. By considering modally equivalent states as
duplicates and thereby transforming the procedure into a graph search,
space and time requirements can be reduced, although great care has to
be taken to deal with cycles correctly.

\leaveout{
It is easy to show that implicitly coordinated policies generalize
implicitly coordinated plans.

\begin{proposition}
  Each implicitly coordinated plan $(a_1, \ldots, a_n)$ for $\Pi =
  \tpl s_0,A,\owner,\gamma \tpr$ has a corresponding implicitly
  coordinated policy for $\Pi$.
\end{proposition}

\begin{proof}[Proof sketch]
  Let $s_i = s_0 \otimes a_1 \otimes \ldots \otimes a_{i}$. Then we
  can construct the policy $\pi$ with $\pi(s'_i) = a_{i+1}$ for all
  $s'_i \in \globals(s_i)$ and all $i=0,\dots,n-1$.  All three
  requirements of Definition~\ref{defglobalpolicy} trivially hold. We
  further have to show that $\pi$ is an implicitly coordinated policy
  for $\Pi$. Finiteness and foundedness are trivial. Closedness
  results from the correspondence to
  Proposition~\ref{slem}.
\end{proof}}
 % Cooperative Epistemic Planning - Conditional
% !TEX root =  paper.tex

\section{Experiments}

We implemented a planner that is capable of finding implicitly
coordinated plans and policies\footnote{Our code can be downloaded
at \url{https://gkigit.informatik.uni-freiburg.de/tengesser/planner}},
and conducted two experiments: one
small case study of the Russian card problem~\cite{ditmarsch03}
intended to show how this problem can be modeled and solved from an
individual perspective, and one experiment investigating the scaling
behavior of our approach on private transportation problems in the style
of Examples \ref{rex} and \ref{letter2}, using instances of increasing
size.
\leaveout{
Our planner is written in C++ and uses breadth-first search with
an approximate bisimulation test that is used for state contraction
and duplicate detection. All experiments were performed on a computer
with a single Intel i7-4510U CPU core.
}
\vspace{-0.3em}

\subsection{Russian Card Problem}
\enlargethispage{1\baselineskip}
In the Russian card problem, seven cards numbered $0, \ldots, 6$ are
randomly dealt to three agents. Alice and Bob get three cards each,
while Eve gets the single remaining card. Initially, each agent only
knows its own cards. The task is now for Alice and Bob to inform each
other about their respective cards using only public announcements,
without revealing the holder of any single card to Eve.
The problem was analyzed and solved from the global perspective
by van Ditmarsch et al.~\cite{ditmarsch06}, and a given protocol was
\emph{verified} from an individual perspective
by {\AA}gotnes et al.~\cite{agotnes10}. %before.
We want to \emph{solve} the
problem from the individual perspective of agent Alice and find an
implicitly coordinated policy for her.
%
%We only allow ontic announcements about hands, not about knowledge.
To
keep the problem computationally feasible, we impose some restrictions on % task ?
the resulting policy, namely that the first action has to be Alice
truthfully announcing five possible alternatives for her own hand,
and that the second one has to be Bob announcing the card Eve is holding.
Without loss of generality, we fix one specific initial hand for Alice, namely
$012$. From a plan for this initial hand, plans for all other initial
hands can be obtained by renaming. For simplicity, we only generate
\emph{applicable} actions for Alice, i.e.\ announcements that include
her true hand $012$. This results in the planning task having a
total of $46376$ options for the first action, and $7$ for the second
action. Still, the initial state $s_0$ consists of $140$ worlds, one
for each possible deal of cards. Agents can only distinguish worlds
where their own hands differ. Alice's designated worlds in her
associated local state of $s_0$ are those four worlds in which she
holds hand $012$.

Our planner needs approximately two hours and $600$MB of memory to come
up with a solution policy. In the solution, Alice first announces her
hand to be one of $012$, $034$, $156$, $236$, and $245$. It can be seen
that each of the five hands other than the true hand $012$ contains at least
one of Alice's and one of Bob's cards, meaning that Bob will immediately be
able to identify the complete deal. Also, Eve stays unaware of the individual
cards of Alice and Bob since she will be able to rule out exactly two of the hands,
with each of Alice and Bob's cards being present and absent in at least
one of the remaining hands. Afterwards, Alice can wait for
Bob to announce that Eve has either card $3$, $4$, $5$ or $6$.
% (which will not tell Eve anything new, either).

\subsection{Mail Instances}

%\enlargethispage{0.5\baselineskip}
Our second experiment concerns the letter passing problem from
Examples~\ref{rex} and \ref{letter2}.
We generalized the scenario to allow an arbitrary
number of agents with an arbitrary undirected neighborhood graph,
indicating which agents are allowed to directly pass letters to each
other. As neighborhood graphs, we used randomly generated
Watts-Strogatz small-world networks~\cite{watts98}, exhibiting
characteristics that can also be found in social
networks. Watts-Strogatz networks have three parameters: The number
$N$ of nodes (determining the number of agents in our setting), the
average number $K$ of neighbors per node (roughly determining the
average branching factor of a search for a plan), and the probability
$\beta$ of an edge being a ``random shortcut'' instead of a ``local
connection'' (thereby influencing the shortest path lengths between
agents). We only generate \emph{connected} networks in order to
guarantee plan existence.

We distinguish between the example domains \textsc{MailTell} and
\textsc{MailCheck}. To guarantee plan existence, in both
scenarios the actions are modeled such as to ensure that the letter
position remains common knowledge among the agents in all reachable
states.
The mechanics of \textsc{MailTell} directly correspond
to those given in Example \ref{rex}. There
is only one type of action, publicly passing the letter to a
neighboring agent while privately informing him about the final
addressee. This allows for sequential implicitly coordinated plans. In
the resulting plans, letters are simply moved along a shortest path to
the addressee.
In contrast, in \textsc{MailCheck}, an agent that has the letter can
only check if he himself is the addressee or not using a separate
action (without learning the actual addressee if it is not him). To
ensure plan existence in this scenario, we allow an agent to pass
on the letter only if it is destined for someone else. Unlike in
\textsc{MailTell}, conditional plans are required here. In a solution
policy, the worst-case sequence of successively applied actions contains
an action passing the letter to each agent at least once. As soon as
the addressee has been reached, execution is stopped.
\begin{table}
\caption{Runtime evaluation for randomly generated \textsc{MailTell} and \textsc{MailCheck} instances}
\label{tbl:mail}
\centering
\begin{subtable}{.5\textwidth}
\caption{\textsc{MailTell}, $K=4$, $\beta=0.1$}
\label{tbl:mailtell}
\centering
\begin{tabular}{l ccccc}
agents & $10$ & $20$ & $30$ & $40$ & $50$ \\ \hline
direct path & $1.4$ & $2.3$ & $3.1$ & $3.7$& $3.6$ \\ % \hline
% created  & $25$ & $116$ & $415$ & $965$ & $1023$ \\
% expanded &  $7$ &  $36$ & $142$ & $347$ &  $359$ \\
% discarded  &  $2$ &  $32$ & $166$ & $455$ &  $443$ \\
time/s
& $0.02$ % $\pm$ 0.00
& $0.14$ % $\pm$ 0.03
& $0.65$ % $\pm$ 0.41
& $2.38$ % $\pm$ 2.78
& $5.02$ % $\pm$ 7.37
\end{tabular}
\end{subtable}%
\begin{subtable}{.5\textwidth}
\caption{\textsc{MailCheck}, $K=2$, $\beta=0.1$}
\label{tbl:mailcheck}
\centering
\begin{tabular}{l ccccc}
agents    & $10$   & $15$   & $20$     & $25$    & $30$ \\ \hline
full path & $10.4$ & $16.1$ & $21.7$   & $27.6$  & $33.2$ \\ % \hline
% created   & $402$  & $2073$ & $8065$   & $35691$ & $113481$ \\
% expanded  & $361$  & $1968$ & $7771$   & $34890$ & $111582$ \\
% discarded & $552$  & $3229$ & $13126$  & $59827$ & $193555$ \\
time/s
&  $0.02$ % $\pm$ 0.01
&  $0.18$ % $\pm$ 0.14
&  $1.14$ % $\pm$ 1.05
&  $8.03$ % $\pm$ 11.52
& $38.76$ % $\pm$ 72.43
\end{tabular}%
\end{subtable}%
\end{table}

% \textsc{MailCheck} $K=4$, $\beta=0.1$
%\begin{tabular}{l ccccc}
%agents    & $7$    & $9$    & $11$    &     $13$ & $15$ \\ \hline
%full path & $7.0$  & $9.0$  & $11.0$  &   $13.0$ & $15.0$ \\ % \hline
% created   & $712$  & $3161$ & $13071$ &  $50104$ & $188997$ \\
% expanded  & $642$  & $3027$ & $12838$ &  $49739$ & $188421$ \\
% discarded & $1859$ & $9167$ & $39528$ & $154756$ & $588582$ \\
%time/s
%&  $0.03$ % $\pm$ 0.01
%&  $0.21$ % $\pm$ 0.02
%&  $1.14$ % $\pm$ 0.14
%&  $5.90$ % $\pm$ 0.91
%& $27.30$ % $\pm$ 5.07
%\end{tabular}%}

%
Experiments were conducted for both scenarios with different parameters (Table
\ref{tbl:mail}).
For finding sequential as well as conditional plans, our
implementation uses breadth-first search over a regular graph and over an AND-OR
graph, respectively. For each set of parameters, 100 trials were performed.
For \textsc{MailTell}, {\em direct path} denotes the average
shortest path length between sender and addressee, while for \textsc{MailCheck},
{\em full path} denotes the average length of a shortest path passing through
all agents starting from the sender.

While the shortest path length between sender and addressee grows very
slowly with the number of agents (due to the shortcut
connections in the network), the shortest path passing through all agents roughly
corresponds to the number of agents.  Since these measures directly
correspond to the minimal plan lengths, the observed exponential
growth of space and time requirements with respect to them (and to the
base $K$) is unsurprising.

Note also that in both scenarios, the number of agents determines the
number of worlds (one for each possible addressee) in the initial
state.  Since the preconditions of the available actions are mutually
exclusive, this constitutes an upper bound on the number of worlds per
state throughout the search. Thus we get only a linear overhead in
comparison to directly searching the networks for the relevant paths.
  % Experiments
\section{Conclusion and Future Work}

We introduced an interesting new cooperative, decentralized
planning concept without the necessity of explicit coordination or
negotiation. Instead, by modeling all possible communication directly
as plannable actions and relying on the ability of the autonomous
agents to put themselves into each other's shoes (using perspective
shifts), some problems can be elegantly solved achieving implicit
coordination between the agents. We briefly demonstrated an
implementation of both the sequential and conditional solution
algorithms and its performance on the Russian card problem and two
letter passing problems.

\enlargethispage{\baselineskip}
Based on the foundation this paper provides, a number of
challenging problems needs to be addressed. First of all, concrete
problems (e.g.\ epistemic versions of established multi-agent planning
tasks) need to be formalized with a particular focus on the
question of which kinds of communicative actions the
agents would need to solve these problems in an implicitly coordinated
way. As seen in the \textsc{MailTell} benchmark, the dynamic epistemic
treatment of a problem does not necessarily lead to more than linear
overhead. It will be interesting to identify classes of tractable
problems and see how agents cope in a simulated environment.
Another issue that is relevant in practice concerns the interplay of
the single agents' individual plans. In our setting, the agents have
to plan individually and decide autonomously when and how to
act. Also, when it comes to action application, there is no predefined
notion of agent precedence. This leads to the possibility of {\em
  incompatible} plans, and in consequence to the necessity for agents
having to {\em replan} in some cases.  While our notion of implicitly
coordinated planning explicitly forbids the execution of actions
leading to {\em dead-end situations} (i.e.\ non-goal states where
there is no implicitly coordinated plan for any of the agents),
replanning can still lead to {\em livelocks}.  Both the conditions
leading to livelocks and individually applicable strategies to avoid
them need to be investigated.

%Since we need to be
%able to deal with replanning anyway, we can follow
%Andersen, Bolander, and Jensen~\cite{andersen13} and also investigate another
%successor function $\sigma_{\textit{plaus}}$ that maps only into the
%{\em most plausible} successor states. We expect this to lead to less
%branching and thus higher efficiency than the successor functions
%defined above.

% TODO:
% while, on the cost side, now possible deadlock situations could result
% in the agents failing an actually solvable problem
  % Conclusion

\nocite{*}
\bibliographystyle{eptcs}
\bibliography{paper}
\end{document}